\documentclass{article} 
\usepackage{iclr2026_conference,times}

\usepackage{amsmath,amsfonts,bm}

\def\eqref#1{equation~\ref{#1}}

\def\1{\bm{1}}

\DeclareMathAlphabet{\mathsfit}{\encodingdefault}{\sfdefault}{m}{sl}
\SetMathAlphabet{\mathsfit}{bold}{\encodingdefault}{\sfdefault}{bx}{n}

\usepackage{hyperref}
\usepackage{url}

\usepackage{algorithm}
\usepackage{algorithmic}
\usepackage{amsthm}
\newtheorem{theorem}{Theorem}
\usepackage{makecell}
\usepackage{amssymb}
\usepackage{booktabs}
\usepackage{multirow}
\usepackage{caption}
\usepackage{subcaption}
\usepackage{graphicx}

\title{RS-ORT: A Reduced-Space Branch-and-Bound Algorithm for Optimal Regression Trees}

\author{Cristobal Heredia Galleguillos \\
Department of Industrial and Management Systems Engineering \\
University of South Florida \\
Tampa, FL, USA \\
\texttt{cristobalheredia@usf.edu} \\
\And
Pedro Chumpitaz-Flores \\
Department of Industrial and Management Systems Engineering \\
University of South Florida \\
Tampa, FL, USA \\
\texttt{pedrochumpitazflores@usf.edu} \\
\And
Kaixun Hua \\
Department of Industrial and Management Systems Engineering \\
University of South Florida \\
Tampa, FL, USA \\
\texttt{khua@usf.edu}
}

\iclrfinalcopy 
\begin{document}

\maketitle

\begin{abstract}
Mixed-integer programming (MIP) has emerged as a powerful framework for learning optimal decision trees. Yet, existing MIP approaches for regression tasks are either limited to purely binary features or become computationally intractable when continuous, large-scale data are involved. Naively binarizing continuous features sacrifices global optimality and often yields needlessly deep trees. We recast the optimal regression-tree training as a two-stage optimization problem and propose Reduced-Space Optimal Regression Trees (RS-ORT)—a specialized branch-and-bound (BB) algorithm that branches exclusively on tree-structural variables. This design guarantees the algorithm's convergence and its independence from the number of training samples. Leveraging the model’s structure, we introduce several bound tightening techniques—closed-form leaf prediction, empirical threshold discretization, and exact depth-1 subtree parsing—that combine with decomposable upper and lower bounding strategies to accelerate the training. The BB node-wise decomposition enables trivial parallel execution, further alleviating the computational intractability even for million-size datasets. Based on the empirical studies on several regression benchmarks containing both binary and continuous features, RS-ORT also delivers superior training and testing performance than state-of-the-art methods. Notably, on datasets with up to 2{,}000{,}000 samples with continuous features, RS-ORT can obtain guaranteed training performance with a simpler tree structure and a better generalization ability in four hours.
\end{abstract}

\section{Introduction}
Decision trees, known for their strong interpretability as a typical supervised learning algorithm, have been widely used across diverse fields  \cite{coscia_automatic_2024}, \cite{xia_application_2025}, \cite{mienye_survey_2024}. Since the initial development of early decision tree algorithms by \cite{morgan_problems_1963}, numerous heuristic methods have been introduced to approach the optimal solution, including ID3 \cite{quinlan_induction_1986}, C4.5 \cite{quinlan_improved_1996}, and CART \cite{breiman_classification_1984}. However, these heuristic approaches are inherently prone to producing sub-optimal solutions. In high-stakes fields such as medicine, where both accuracy and interpretability are crucial, the use of optimal decision trees is particularly valuable. These models provide transparent predictions that can be more easily validated and trusted by medical professionals \cite{rudin_stop_2019}.

Although a decision tree can be formulated as an optimization problem to find the most accurate and compact structure, solving it at optimality is computationally challenging —even for binary classification, since the task is NP-complete \cite{hyafil_constructing_1976}. Despite the computational complexity, recent studies have demonstrated the benefits of pursuing global optimality. For instance, \cite{bertsimas_optimal_2017} shows that globally optimal decision trees can achieve average absolute improvements in R\textsuperscript{2}. In this work, “optimal” refers to the globally optimal solution of a regularized empirical-risk minimization problem, where the objective combines mean-squared error with a penalty on the number of active internal nodes under a fixed maximum depth constraint. Note that the term “optimal” has been used with different meanings in prior work, such as zero-error trees, minimum-depth formulations, or regularized loss minimization. Our study focuses on the third class, where both accuracy and model structure are jointly optimized.

Decision trees can be broadly categorized into classification and regression trees, depending on the nature of the predictive task. Both types share a similar structure: they consist of branching nodes that recursively split the feature space, and leaf nodes that assign predictions. The key difference lies in the output of the leaf nodes, where classification trees output a discrete class label (typically an integer), whereas regression trees assign a continuous real-valued prediction. Regression trees pose unique challenges due to the continuous nature of their outputs. In particular, computing the optimal regression tree becomes significantly more complex than in the classification setting, as the search space grows with the granularity of the continuous target variable. 

In classification, multiple methods have been proposed to find the optimal solution. Using mixed-integer programming (MIP), several approaches have been developed over time \cite{verwer_learning_2017, verwer_learning_2019, gunluk_optimal_2019, aghaei_learning_2020, gunluk_optimal_2021, liu_bsnsing_2022, boutilier_optimal_2023, aghaei_strong_2024, liu_optimal_2024}. Some variants leverage mixed-integer quadratic programming (MIQP) \cite{donofrio_margin_2024, burgard_mixed-integer_2024}. Other strategies are inspired by itemset mining techniques \cite{nijssen_mining_2007, aglin_learning_2020}. Tailed branch-and-bound and dynamic programming have also been used to enhance scalability and performance \cite{hu_optimal_2019, lin_generalized_2020, mctavish_fast_2022, demirovic_murtree_2022, hua_scalable_2022}; \cite{Pascal_2024}; \cite{brita_optimal_2025}. More recent methods incorporate AND/OR graph search under a Bayesian inference framework \cite{sullivan_maptree_2024, chaouki_branches_2025}. Finally, satisfiability-based (SAT) formulations have been explored \cite{narodytska_learning_2018, avellaneda_efficient_2020, hu_learning_2020} \cite{janota_sat-based_2020}; \cite{shati_sat-based_2021}; \cite{Schidler_Szeider_2021}.

In contrast, optimal regression trees have received significantly less attention due to the added complexity of continuous outputs, which substantially increase the search space and limit the applicability of classification-based techniques. The formulation of optimal nonlinear regression trees as a mixed-integer program (MIP) was first introduced in~\cite{bertsimas2019ml}, where the authors proposed a method that guarantees globally optimal solutions. To improve scalability, a gradient descent-based method was later developed to efficiently find near-optimal solutions~\cite{bertsimas_near-optimal_2021}. \cite{zhang_optimal_2023} introduced a dynamic programming with bounds (DPB) algorithm capable of computing provably optimal sparse regression trees and can scale to datasets with 2 millions of samples, which is the current SOTA in the field of optimal regression trees. 
However, their approach requires fully binarized inputs or to transform continuous input into binary, which limits applicability to real-valued or high-cardinality features. Other methods make similar trade-offs. \texttt{evtree} \cite{grubinger_evtree_2014} uses evolutionary search but lacks guarantees. \cite{bos_piecewise_2024} propose dynamic programming methods for constant and linear regression trees, leveraging depth-two optimizations and per-instance cost precomputations to scale, though still constrained to binarized feature spaces. \cite{roselli_experiments_2025} use MIP to compute optimal model trees with SVM leaves, but restrict tree size and number of splits to ensure tractability.

In this work we overcome key scalability limitations in optimal regression-tree learning by introducing three core contributions. First, we formulate regression tree training as a two-stage optimization model capable of processing continuous features directly. Unlike methods that rely on feature discretization, our formulation operates directly on continuous features without preprocessing. This enables the construction of compact, interpretable trees that maintain competitive prediction accuracy. The first-stage variables encode only the tree structure (i.e., split features, thresholds, and leaf assignments), while the second stage captures sample-specific responses. By building on the reduced-space branch-and-bound (BB) framework of~\citet{cao_scalable_2019} we prove that branching \textit{exclusively} on the first-stage variables guarantees finite convergence, even though the second stage includes both binary and continuous decisions. This design ensures that the search space—and thus the memory footprint—remains \textit{independent} of the number of training samples, a property not shared by previous MIP formulations for regression trees.

Secondly, we develop a series of tight bounding mechanisms that turn this theoretical advantage into practical speed. For any fixed routing pattern, we derive closed‑form least‑squares predictions for each leaf, eliminating an entire block of continuous decision variables from the BB nodes (i.e., no need of branching on leaf label variables). We further discretize candidate split thresholds to the empirical feature values without loss of generality, which shrinks the domain of each branching variable. 

Finally, we fathom BB nodes in advance by determining all variables of the decision tree’s last depth‑1 subtrees through exact greedy search after the fix of upper level tree structure variables, which provide a stronger lower bound that accelerates the convergence. Combined with our decomposable bounds, the resulting algorithm is naturally parallelizable and scales to multi-million-row datasets with provable training quality guarantee.

On several benchmarks, RS‑ORT attains strictly lower training loss and better test RMSE than state‑of‑the‑art MIP and heuristic baselines. Remarkably, for a dataset with 2 million continuous samples RS‑ORT reaches a provably globally optimal tree—often two to three levels shallower than competing models—within four hours, all while reporting optimality gaps below $0.1$.

\section{Decomposable Optimal Regression Tree Problem} 
This work addresses the training of optimal decision trees for regression task involving datasets with continuous input features. The formulation presented here follows the skeleton of \cite{bertsimas2019ml} and~\cite{hua_scalable_2022}.

Let $\mathcal{N} = \{1, \dots, n\}$ and $\mathcal{P} = \{1, \dots, P\}$ denote the index sets of the $n$ observations and $P$ input features, respectively. Given a dataset $X = \{x_i \mid x_i \in \mathbb{R}^P, i \in \mathcal{N}\}$ and its corresponding response set $Y = \{y_i \mid y_i \in \mathbb{R}, i \in \mathcal{N}\}$. We aim to learn an optimal regression tree model $F : \mathbb{R}^P \rightarrow \mathbb{R}$ of maximum depth $D$, such that the following objective function is minimized:

\begin{equation}
\mathcal{L}(F) \;=\; \sum_{i \in \mathcal{N}} \ell\bigl(F(x_i),\, y_i\bigr) \;+\; \lambda\, R(F),
\label{eq:loss}
\end{equation}

Here, $\ell(\cdot,\cdot)$ denotes the loss function, which measures the discrepancy between predicted and true values. In this work, we use the mean squared error (MSE) as the loss function. $R(\cdot)$ is the regularization function that quantifies the complexity of the decision tree model. The parameter $\lambda > 0$ controls the trade-off between model fit and complexity: smaller values of $\lambda$ result in deeper trees by enabling splits with lower gain, while larger values restrict tree growth by requiring a greater reduction in loss to justify a new branch.

Let $t \in {1, \dots, T}$ be the index of a node in the tree, where $T = 2^{D+1} - 1$ is the total number of nodes for a tree of depth $D$. As in~\cite{bertsimas_regression_2017}, we define $p(t) = \lfloor t/2 \rfloor$ as the parent of node $t$, $A_L(t)$ as the set of ancestors whose \emph{left} branch is followed on the path from the root to $t$, and $A_R(t)$ as the set of ancestors whose \emph{right} branch is followed.

The set of nodes is partitioned into two disjoint subsets: the \emph{decision nodes} $\mathcal{T}_D = {1,\dots,\lfloor T/2 \rfloor}$ and the \emph{leaf nodes} $\mathcal{T}_L = {\lfloor T/2 \rfloor + 1,\dots, T}$.
Training the optimal regression tree can therefore be formulated as the following optimization problem:

\begin{subequations}\label{prob:main_reg}
   \begin{align}
    \min_{a,b,c,d} \quad & \sum_{i=1}^n \left(\frac{L_{i*}}{\hat{L}} + \frac{\lambda}{n} \sum_{t \in \mathcal{T}_D} d_t \right)\\
    \text{s.t.} \quad & L_{i*} \geq (f_i - y_i)^2, \forall i \in \mathcal{N} \label{con:reg_b}\\
    & -M_f(1 - z_{it}) \leq f_i - c_{t} \leq M_f(1 - z_{it}),\; \forall t \in \mathcal{T}_L, \forall i \in \mathcal{N} \label{con:reg_c}\\
    & \sum_{t \in \mathcal{T}_L} z_{it} = 1, \forall i \in \mathcal{N} \label{con:reg_d}\\
    & \mathbf{a}_m^\top (x_i + \mathbf{\epsilon} - \epsilon_{\min}) + \epsilon_{\min} \leq b_m + (1 + \epsilon_{\max})(1 - z_{it}), \;\forall m \in A_L(t),\, \;\forall t \in \mathcal{T}_L, \;\forall i \in \mathcal{N} \label{con:reg_e}\\
    & \mathbf{a}_m^\top x_i \geq b_m - (1 - z_{it}), \;\forall m \in A_R(t),\, \;\forall t \in \mathcal{T}_L, \;\forall i \in \mathcal{N} \label{con:reg_f}\\
    & \sum_{j=1}^P a_{jt} = d_t, \;\forall t \in \mathcal{T}_D \label{con:reg_g}\\
    & 0 \leq b_t \leq d_t, \;\forall t \in \mathcal{T}_D \label{con:reg_h}\\
    & d_t \leq d_{p(t)},\;\forall t \in \mathcal{T}_D \label{con:reg_i}\\
    & a_{jt}, d_t \in \{0,1\}, \;\forall t \in \mathcal{T}_D, \;\forall j\in \mathcal{P}\; \label{con:reg_j}\\
    & c_{t} \in \mathbb{R}, \;\forall t \in \mathcal{T}_L \label{con:reg_k}\\
    & z_{it} \in \{0,1\}, \;\forall t \in \mathcal{T}_L, \;\forall i \in \mathcal{N} \label{con:reg_l}\\
    & L_{i*}, f_i \in \mathbb{R}, \;\forall i \in \mathcal{N} \label{con:reg_m}
    \end{align} 
\end{subequations}

The structure of the regression tree is defined by variables \( a, b, c, d \).  Prior to optimization, all data input features are scaled to the interval \( [0, 1] \) to ensure consistent bounds across variables.  Each decision node \( t \in \mathcal{T}_D \) is associated with a binary variable \( d_t \in \{0,1\} \), which indicates whether the node performs a split (i.e., is active).  When active, the node selects a single feature via the binary vector \( \mathbf{a}_t = [a_{1t}, \dots, a_{Pt}]^\top \in \{0,1\}^P \), and applies a threshold \( b_t \in [0,1] \) to determine the branching condition. In Constraint~\ref{con:reg_e}, \(\boldsymbol{\epsilon}\) is a vector where each component \(\epsilon_j\) corresponds to the smallest nonzero difference between two consecutive sorted values of feature \(j\). From this, we define \(\epsilon_{\min}\) as the smallest of these gaps across all features, and \(\epsilon_{\max}\) as the largest.
Each leaf node \( t \in \mathcal{T}_L \) is assigned a continuous prediction value \( c_t \in \mathbb{R} \). To guarantee that each leaf node \(t\) produces a consistent prediction \(c_t\), Constraint~\ref{con:reg_c} enforces this relation. Here, \(f_i\) denotes the fitted value assigned by the regression tree to sample \(i\). The parameter \(M_f\) is a sufficiently large constant that bounds the difference \(f_i - c_t\). Since at optimality both the fitted values \(f_i\) and the leaf predictions \(c_t\) lie within the observed response range, we set \(M_f = \max_i y_i - \min_i y_i\).

To model the routing of samples through the tree, we introduce binary variables \( z_{it} \in \{0,1\} \), which indicate whether sample \( i \in \mathcal{N} \) is assigned to leaf node \( t \in \mathcal{T}_L \). A routing constraint Constraint~\ref{con:reg_d} ensures that each sample is assigned to exactly one leaf.  Based on this assignment, each sample receives a fitted value \( f_i \in \mathbb{R} \) corresponding to the prediction of its associated leaf. A key difference from \cite{bertsimas_regression_2017} is that, while their formulation enforces a minimum leaf size constraint, our model omits it and instead regulates tree complexity through the regularization term \(\tfrac{\lambda}{n}\sum_{t \in \mathcal{T}_D} d_t\). Hence, splits that yield only marginal error reduction are automatically discouraged when \(\lambda\) is large, making an explicit minimum leaf size unnecessary.

Finally, the squared error between the predicted value \( f_i \) and the ground truth \( y_i \) is upper bounded by an auxiliary variable \( L_{i*} \in \mathbb{R} \), as specified in Constraint~\ref{con:reg_b}. The value of $\hat{L}$ corresponds to the baseline mean squared error when predicting all training targets with their average.

\section{Size of the Search Space}
We quantify the number of distinct tree structures, that is, unique combinations of feature and split-threshold assignments across all internal nodes induced by a full binary tree of bounded depth. The bound below illustrates a combinatorial explosion in the search space even for shallow trees, which renders exhaustive enumeration impractical.
\begin{theorem}[Upper bound on distinct tree structures]\label{thm:search-space}
Let $n$ be the number of samples, $P$ the number of features, and $D$ the maximum depth.
Assume that each split threshold is chosen from among the distinct feature values present in the training data.
Then the number of distinct tree structures is at most
\[
N_{\mathrm{struct}}
\;\le\;
P^{\,2^D - 1}\;
\prod_{t=0}^{D-1}
\Bigl(\big\lfloor \tfrac{n}{2^t} \big\rfloor - 1 \Bigr)^{2^t}.
\]
\end{theorem}

\begin{proof}
A full binary tree of maximum depth $D$ has $2^D-1$ internal nodes. Each internal node selects a feature $j\in\{1,\dots,P\}$ independently, yielding at most $P^{2^D-1}$ possible feature assignments.

Fix a level $t\in\{0,\dots,D-1\}$ and let $n_u$ be the number of samples reaching node $u$ at that level. For node $u$, the distinct thresholds induced by its local sample multiset are at most $(n_u-1)$ (between consecutive ordered values). Since there are $2^t$ nodes at level $t$ and $\sum_u n_u=n$, the product of threshold choices at that level is maximized when the $n_u$ are as balanced as possible; hence it is bounded by $\bigl(\lfloor n/2^t\rfloor-1\bigr)^{2^t}$. Taking the product over $t=0,\dots,D-1$ gives
\[
\prod_{t=0}^{D-1}\Bigl(\big\lfloor \tfrac{n}{2^t}\big\rfloor-1\Bigr)^{2^t}
\] 
as an upper bound on the number of threshold combinations for any fixed feature assignment.

Combining the $P^{2^D-1}$ feature choices with the bound on thresholds yields the stated bound in Theorem~\ref{thm:search-space}.

\end{proof}

For depth $D=2$ (three internal nodes),
\[
N_{\mathrm{struct}} \;\le\;
P^{3}\,(n-1)\,\Bigl(\,\big\lfloor \tfrac{n}{2} \big\rfloor - 1\Bigr)^{2},
\]

Considering a small dataset such as the \emph{Concrete} dataset \citep{concrete1998} ($n = 1030$, $P = 8$) and a depth limit of $D = 2$, Theorem~\ref{thm:search-space} yields
\[
N_{\mathrm{struct}} \;\le\;
P^{3}\,(n-1)
\Bigl(\,\big\lfloor\tfrac{n}{2}\big\rfloor-1\,\Bigr)^{2}
= 8^{3}\cdot1029\cdot514^{2}
= 139{,}191{,}134{,}208
\;\approx\; 1.39\times10^{11}.
\]

Even for this small dataset and shallow depth, the number of distinct feature–threshold assignment patterns 
already exceeds $10^{11}$, making exhaustive enumeration infeasible. This magnitude highlights the combinatorial explosion of possible structures and reinforces the need for a pruning-based optimization algorithm such as RS-ORT.

\section{Reduced-space Branch-and-Bound Algorithm}
Following the framework of~\cite{hua_scalable_2022} on large-scale optimal classification tree learning, we can reformulate Problem~\ref{prob:main_reg} as a two-stage optimization model: the first stage determines the tree structure variables (i.e., $a,b,c,d$), while the second stage accounts for sample-dependent variables, such as leaf assignment (i.e., $z$) and loss calculation (i.e., $f,L$).

Therefore, Problem~\ref{prob:main_reg} can be reformulated as solving the following optimization problem:

\begin{equation}
    f(M_0) = \min_{m \in M_0}\sum_{i\in\mathcal{N}} Q_i(m)
\end{equation}
where we denote $m = (a,b,c,d)$ as all first-stage variables and $Q_i(\cdot)$ represents the optimal value of the second-stage problem:
\begin{subequations}\label{eqn:sec_stg}
\begin{align}
    Q_i(m) &= \min_{z_i, L_{i*}} \frac{L_{i*}}{\hat{L}} + \frac{\lambda}{n}\sum_{t\in\mathcal{T}_D}d_t \\ 
    \rm{s.t.} \;\;& \mbox{Constraint}~\ref{con:reg_b}~-~\ref{con:reg_m}
\end{align}
\end{subequations}

Let \( M_0 := \{ m = (a, b, c, d) : m^{\mathrm{lb}} \leq m \leq m^{\mathrm{ub}} \} \) denote the box-constrained domain of the first-stage variables, where \( m^{\mathrm{lb}}, m^{\mathrm{ub}} \) denote lower and upper bounds respectively. At each node of the branch-and-bound tree, we consider a subregion \( M \subseteq M_0 \) and solve the restricted problem \( f(M) = \min_{m \in M} \sum_{i \in \mathcal{N}} Q_i(m) \).

The specific structure of the problem enables a reduced-space branch-and-bound (BB) algorithm, where branching occurs only on the first-stage variables (i.e., tree structure variables \( a, b, c, d \)), while still providing guarantees on the convergence of the algorithm (deferred to Appendix~\ref{apdx:alg_conv}). This design is critical for scalability without sacrificing global optimality, as it keeps the dimensionality of the branch-and-bound search space independent of the dataset size, in contrast to traditional formulations that introduce branching variables that scale with the number of samples.

At each BB node, lower bounds are computed by relaxing the non-anticipativity constraints, which allows to each sample to use a potentially different tree structure rather than enforcing a single shared tree for all samples~\cite{hua_scalable_2022}. This relaxation decouples the problem across samples and enables efficient parallel computation of the lower bound. In contrast, upper bounds are obtained via feasible solutions derived from heuristics or from the solutions of individual lower bounding subproblems. The decomposable nature of the master problem can enable the lower and upper bound calculation through a trivial parallelism to enhance the efficiency of the BB algorithm. Algorithm~\ref{alg:bb_sche} depicts the details of our BB algorithm.

\begin{algorithm}[t]
\caption{Reduced-space Branch and Bound \cite{hua_scalable_2022}} \label{alg:bb_sche}
\begin{algorithmic}[1]
\REQUIRE Initial region $M_0$, tolerance $\epsilon$
\STATE Set iteration index $t \leftarrow 0$, $\mathbb{M} \leftarrow \{M_0\}$
\STATE Initialize bounds: $\alpha_t \leftarrow \alpha(M_0)$, $\beta_t \leftarrow \beta(M_0)$
\WHILE{$\mathbb{M} \neq \emptyset$}
    \STATE Select $M \in \mathbb{M}$ with $\beta(M) = \beta_t$
    \STATE $\mathbb{M} \leftarrow \mathbb{M} \setminus \{M\}$, $t \leftarrow t + 1$
    \STATE Partition $M$ into $M_1$ and $M_2$ (see Appendix~\ref{apdx:bb_bch})
    \STATE Add $M_1$, $M_2$ to $\mathbb{M}$
    \FOR{$j \in \{1,2\}$}
        \STATE Compute $\alpha(M_j)$, $\beta(M_j)$
        \IF{$\beta_i(M_j)$ infeasible for some $i \in \mathcal{N}$}
            \STATE Remove $M_j$ from $\mathbb{M}$
        \ENDIF
    \ENDFOR
    \STATE $\beta_t \leftarrow \min \{\beta(M') \mid M' \in \mathbb{M}\}$
    \STATE $\alpha_t \leftarrow \min \{\alpha_{t-1}, \alpha(M_1), \alpha(M_2)\}$
    \STATE Remove $M'$ from $\mathbb{M}$ if $\beta(M') \geq \alpha_t$
    \IF{$|\beta_t - \alpha_t| \leq \epsilon$}
        \STATE \textbf{stop}
    \ENDIF
\ENDWHILE
\end{algorithmic}
\end{algorithm}

Algorithm~\ref{alg:bb_sche} outlines the reduced-space branch-and-bound (BB) procedure used in RS-ORT. Starting from the initial feasible region $M_0$, the algorithm iteratively partitions the search space into subregions $(M_1, M_2)$ defined over the tree-structure variables. For each subregion, lower and upper bounds are computed and infeasible or dominated regions are pruned. The process continues until the global lower and upper bounds $\alpha_t$ and $\beta_t$ converge within a specified tolerance $\epsilon$, at which point the algorithm terminates with a globally optimal solution.

The effectiveness of this strategy depends not only on computational efficiency but also on its ability to guarantee global optimality. Fortunately, theoretical results support this convergence. \cite{cao_scalable_2019} established that, for a broad class of two-stage problems with continuous recourse, branching solely on first-stage variables suffices to guarantee global convergence. Although the decision tree training problem involves both binary and continuous variables in the second stage, \cite{hua_scalable_2022} proved that the convergence guarantee remains valid for optimal classification trees. An analogous result holds for optimal regression tree problems.

\begin{theorem}\label{thm:bb_cvg}
    Algorithm~\ref{alg:bb_sche} converges to the global optimum by only branching on variable $a,b,c,d$, in the sense that
    \begin{equation}
        \lim \limits_{t\rightarrow\infty}\alpha_t = \lim \limits_{t\rightarrow\infty}\beta_t = f
    \end{equation}
\end{theorem}

The proof relies on the observation that all critical structural decisions are encoded in the first-stage variables. Once the tree structure is fixed, the second-stage decisions—such as sample allocation—can be optimally determined. A detailed proof is provided in Appendix~\ref{apdx:bb_thm}.

Despite the similarity of the BB framework to the optimal classification tree problem, many of the bound-tightening strategies used in that setting are not directly applicable to the regression counterpart. For instance,~\cite{hua_scalable_2022} introduces a sample reduction technique in the classification setting, which allows certain samples to be identified in advance and excluded from subproblem computations. However, this strategy is not applicable in the regression context. Consequently, we introduced new bound-tightening strategies tailored to the regression setting to improve the convergence rate of the proposed BB framework.

\section{Bound Tightening Strategies}
To further enhance the scalability of the reduced-space branch-and-bound (BB) algorithm applied to regression problems, we introduce three core improvements to its base formulation.

\subsection{Leaf Predictions as Implicit Variables}

In regression trees, each leaf \( t \in \mathcal{T}_L \) is associated with a continuous prediction value \( c_t \in \mathbb{R} \). Once the tree structure—namely the split directions \(A\), thresholds \(b\), and split indicators \(d\)—is fixed, each sample is routed along a unique path and assigned to a single, specific leaf based on its features. Under squared-error loss, and given a fixed tree structure and sample-to-leaf assignments, the optimal value of each \( c_t \) admits a closed-form solution. This enables the branch-and-bound algorithm to treat leaf predictions as implicit variables—computed only when needed—without sacrificing optimality.

\begin{theorem}[Closed-form optimality of leaf predictions]
\label{thm:leaf_closed_form}

Let \( \mathcal{S}_t = \{ i \in \mathcal{N} : z_{it} = 1 \} \) denote the set of samples assigned to leaf \( t \in \mathcal{T}_L \), and assume the routing variables \( z \) are fixed. Under the squared-error loss, the unique value of \( c_t \) that minimizes
\[
\sum_{i \in \mathcal{S}_t} (y_i - c_t)^2
\]
is given by the empirical mean of its targets:
\[
c_t^\star = \frac{1}{|\mathcal{S}_t|} \sum_{i \in \mathcal{S}_t} y_i.
\]

\end{theorem}

Proof can be found in Appendix~\ref{apdx:proof-error-decomp}. As a consequence, the prediction variables \( c_t \) can be excluded from the branch-and-bound decision space, since they admit closed-form computation once the routing variables are fixed. Instead of including \( c_t \) as decision variables, their values can be derived analytically at each node of the search tree after the routing has been determined. This decouples the prediction step from the combinatorial structure of the tree, simplifying the optimization problem.

This modeling choice offers two key advantages. First, it reduces the dimensionality of the search space, as removing the \( |\mathcal{T}_L| \) continuous variables from the branch-and-bound procedure leads to an exponential reduction in the number of nodes explored. Second, optimality is preserved: because \( c_t^\star \) is the unique minimizer for any fixed routing, deferring its computation does not affect the correctness of the solution or exclude any globally optimal configuration.

\subsection{Discretization of Threshold Variables}
Generally, if two consecutive sample values are \(v_{(k)} \le v_{(k+1)}\), then every \(b \in (v_{(k)}, v_{(k+1)})\) induces the same partition. Hence the only \emph{distinct} splits occur when \(b\) equals an observed feature value. Restricting each \(b_t\) to the empirical set \(\{x_{ij}\}_{i=1}^{n}\) is therefore \emph{loss-free}: it preserves every unique split while converting an uncountable domain into a finite one. For example, let a single feature column \(x_{\bullet j}\) take the (sorted) values \(0.18,\,0.23,\,0.41,\,0.55,\,0.55,\,0.77\). For any threshold \(b\) that falls strictly between \(0.23\) and \(0.41\), all observations satisfy either \(x_{ij} \le b\) or \(x_{ij} > b\) in exactly the same way. The split imposed on the dataset is therefore identical for, say, \(b = 0.24\) and \(b = 0.40\).

In our implementation of BB, each threshold variable \(b_t\) is associated with a dynamic lower bound \(b_t^l\) and upper bound \(b_t^u\), which delimit the current feasible interval for splitting. To determine a split point within this interval, we first identify the features that remain active at node \(t\). For each such feature, we inspect its sorted sample values and locate the first value that is not smaller than \(b_t^l\), and the last one that is not greater than \(b_t^u\). We then select the median index within this range to define the candidate split value \(b_t\).

If the interval still contains \emph{two or more} distinct sample values, the midpoint index produces a candidate \(\hat{b} \in \{x_{ij}\}\) that exactly bisects the remaining options. The BB tree then spawns two children: \(b_t \le \hat{b}\) (left branch) and \(b_t > \hat{b}\) (right branch). If the interval has shrunk to a \emph{single} sample value, further branching on \(b_t\) is unnecessary, so the \(b_t\) value becomes determined for that branch.

Because each branch removes at least one distinct feasible split, the discretised strategy guarantees progress and reduces the continuous search to a finite set of cases without impairing completeness or optimality.

\subsection{Closed-form Evaluation of Terminal Splits}

Consider a full binary tree of target depth \(D\). For any leaf at depth \(D\), denote its parent by \(P \in \mathcal{T}_{D-1}\) and its grandparent by \(G = p(p(P)) \in \mathcal{T}_{D-2}\).

Once the split variables \((A_G, b_G, d_G)\) of every grandparent \(G\) at depth \(D-2\) have been fixed, each sample is routed unambiguously to exactly one parent node \(P\). Consequently, the data matrices \(X_P = \{x_i : z_{iP}=1\}\) and \(Y_P = \{y_i : z_{iP}=1\}\) are fully determined and remain unchanged in every descendant BB node. By the decomposability result in Appendix~\ref{apdx:proof-error-decomp}, the contribution of a depth-\(D\) subtree rooted at \(P\) decomposes as
\[
\underbrace{\sum_{i \in P}(y_i - \hat{y}_P)^2}_{\text{parent as leaf}} - \left[\sum_{i \in L}(y_i - \hat{y}_L)^2 + \sum_{i \in R}(y_i - \hat{y}_R)^2\right],
\]
which depends only on samples assigned to \(P\). The decision to split \(P\) or keep it as a leaf can thus be made \emph{locally}.

For every parent node \(P\), we fit a regression tree of maximum depth one on the subset \((X_P, Y_P)\). This operation corresponds to training a CART model with \texttt{max\_depth = 1}, which guarantees that the selected split is globally optimal for trees of that depth, as it maximizes the reduction in error denoted by \(\Delta(P)\). We then compare this reduction \(\Delta(P)\) with a regularization penalty given by \(\lambda |P| / \hat{L}\), where \(|P|\) is the number of samples in node \(P\). If the gain \(\Delta(P)\) exceeds the penalty, we accept the split; otherwise, node \(P\) remains a leaf in the tree. Since data routed to \(P\) is fixed, this finalizes the entire depth-\(D\) decision beneath \(P\). No further branching is required for \(P\) or its children.

Resolving every parent node at depth \(D-1\) using the CART procedure described above has three key effects on the branch-and-bound (BB) search. First, it prunes entire branches of the BB tree, effectively removing two leaves per resolved parent node. Second, it tightens the upper bounds by evaluating the corresponding subtrees at their optimal configuration. Third, it preserves global optimality, since CART identifies the exact minimizer of the local objective at each parent node.

\section{Experiments}
\label{sec:experiments}

We evaluated our Reduced-Space Branch-and-Bound Optimal Regression Tree (RS-ORT) model on several regression datasets of varied scale. We used regression trees of fixed depth 2 (except where noted) to ensure comparability and computational feasibility across methods, and set time limits of 4 hours per dataset.

\begin{itemize}
  \item \textbf{Bertsimas \& Dunn (2019)} – We implemented the MIQP formulation from \cite{bertsimas2019ml}, using CPLEX 22.1.2 as solver. All parameters followed the original publication.
  
  \item \textbf{CART} – We used the \texttt{DecisionTreeRegressor} from \texttt{scikit-learn (v1.7.1)}, with \texttt{max\_depth} set to match RS-ORT for a fair comparison. CART was used both as a baseline and to generate warm-start solutions.

  \item \textbf{OSRT} – We evaluated the Optimal Sparse Regression Tree (OSRT) algorithm by \cite{zhang_optimal_2023} using their official Python implementation \url{https://github.com/ruizhang1996/optimal-sparse-regression-tree-public}.
\end{itemize}

We applied a consistent random sampling scheme across all methods to ensure comparable train/test splits. All models were subject to a fixed wall-clock time limit of four hours. Experiments with the RS-ORT method were conducted in high-performance computing environments, using between 40 and 200 physical CPU cores for most datasets. For large-scale datasets (over 100,000 observations), up to 1,000 cores were used. For OSRT (public implementation) and MIQP baseline supports multi-threading but not distributed-memory execution; since these baselines cannot aggregate memory across nodes, we ran them on the largest single-node machines available to us (20 threads, 128 GB RAM).

Our method was implemented in \texttt{Julia} and parallelized via multi-worker distributed execution, assigning one worker per physical core. The architecture enables RS-ORT to scale to massive datasets: we successfully solved the \emph{Household} dataset with over 2 million rows and continuous-valued features, achieving a provably optimal regression tree within four hours. Based on the existing literature, this is the first exact method capable of handling continuous features at this scale. 

\subsection{Datasets}
We benchmarked our algorithms on ten publicly available regression tasks drawn from the \emph{UCI Machine Learning Repository} \cite{abalone1995, airfoil2014, ccpp2014, concrete1998, microgasturbine2024, seoulbike2020, steel_industry_851, household2012}, \emph{OpenML} \cite{cpuact_openml}, and \emph{Kaggle} \cite{insurance_kaggle}. The collection spans a broad spectrum of domains—including civil engineering (\emph{Concrete}), power generation (\emph{CC Power Plant}), and urban demand forecasting (\emph{Seoul Bike}), among others—with sample sizes ranging from 414 to 2,075,259 observations and feature dimensionalities between 4 and 21. Overall, this suite provides a balanced mix of small-, medium-, and large-scale problems, heterogeneous target variables, and varying degrees of feature complexity, making it a robust test bed for evaluating optimal decision tree methods. 

\subsection{Test Performance}

Each method is trained on \(70\%\) of the data and evaluated on the remaining \(30\%\).
We report two kinds of quantities: 
(i) \emph{Train/Test RMSE}, computed as the unregularized root mean squared error on the training and testing splits, respectively; and 
(ii) \emph{Gap (\%)}, defined as the solver’s relative optimality gap on its own regularized training objective at termination. 
RMSE serves as a fair and comparable metric across methods, while \emph{Gap (\%)} quantifies convergence with respect to each solver’s internal objective. Notably, our objective function is trivially transferable to RMSE, ensuring consistent and fair evaluation across solvers.

Because all methods are deterministic optimization formulations solved under identical data splits and solver settings, the reported metrics are directly comparable and not subject to random variability. For this reason, we report exact RMSE and optimality gaps instead of statistical significance tests. As CART is a heuristic method without optimality guarantees, no gap is reported.

\begin{table*}[t]
\centering
\footnotesize
\setlength{\tabcolsep}{4pt} 
\caption{Performance on continuous datasets. “OoM” = out-of-memory error; “TOWT” = time-out without producing a valid tree.  “-” = not reported (CART has no gap and negligible runtime).}
\label{tab:test-results-continuous}
\begin{tabular}{cccccccc}
\toprule
Dataset & $n$ & $P$ & Method & Train RMSE & Test RMSE & Gap (\%) & Time (s) \\
\midrule
\multirow{4}{*}{Concrete } & \multirow{4}{*}{1,030} & \multirow{4}{*}{8}
  & RS-ORT      & \textbf{11.96} & \textbf{11.80} & \textbf{$<$0.01}  & 631.00  \\
 & & & Bertsimas  & \textbf{11.96} & \textbf{11.80} & \textbf{$<$0.01} & 10046.68 \\
 & & & OSRT       & \textbf{11.96} & \textbf{11.80} & \textbf{$<$0.01}  & \textbf{115.67} \\
 & & & CART       & 12.01 & 12.57 & --       & --    \\
\midrule
\multirow{4}{*}{Insurance } & \multirow{4}{*}{1,338} & \multirow{4}{*}{6}
  & RS-ORT      & 5168.28 & 4733.58 & $<$0.01   & 87.20   \\
 & & & Bertsimas  & 5168.28 & 4733.58 & $<$0.01  & 390.76 \\
 & & & OSRT       & 5168.28 & 4733.58 & $<$0.01   & \textbf{26.06}  \\
 & & & CART       & 5168.28 & 4733.58 & --        & --    \\
\midrule
\multirow{4}{*}{Airfoil Noise } & \multirow{4}{*}{1,503} & \multirow{4}{*}{5}
  & RS-ORT      & 5.39  & 5.26  & \textbf{$<$0.01}   & 196.00  \\
 & & & Bertsimas  & 5.39  & 5.26  & 50.90  & 1158.31 \\
 & & & OSRT       & 5.39  & 5.26  & \textbf{$<$0.01}   & \textbf{4.63}   \\
 & & & CART       & 5.39  & 5.26  & --        & --   \\
\midrule
\multirow{4}{*}{Abalone} & \multirow{4}{*}{4,177} & \multirow{4}{*}{8}
  & RS-ORT      & 2.56  & 2.52  & \textbf{$<$0.01}   & 6012.00 \\
 & & & Bertsimas  & 2.56  & 2.52  & 99.98       & 14400.00 \\
 & & & OSRT       & 2.56  & 2.52  & \textbf{$<$0.01}   & \textbf{910.64} \\
 & & & CART       & 2.56  & 2.52  & --        & --     \\
\midrule
\multirow{4}{*}{CPU ACT} & \multirow{4}{*}{8,192} & \multirow{4}{*}{21}
  & RS-ORT      & \textbf{5.99}  & 6.03  &  \textbf{8.08}  & 14400.00  \\
 & & & Bertsimas  & 6.01  & 6.03  & 100.00       & 14400.00 \\
 & & & OSRT       & OoM  & OoM  & OoM   & OoM \\
 & & & CART       & 6.01  & 6.03  & --        & --     \\
\midrule
\multirow{4}{*}{Seoul Bike} & \multirow{4}{*}{8,760} & \multirow{4}{*}{12}
  & RS-ORT        & 478.67  & 495.92  & \textbf{$<$0.01}   & \textbf{10116.00} \\
 & & & Bertsimas  & 478.67  & 495.92  & 100.00 & 14400.00 \\
 & & & OSRT       & 478.67    & 495.92    & \textbf{$<$0.01}   & 11606.78 \\
 & & & CART       & 478.67  & 495.92  & --      & --     \\
\midrule
\multirow{4}{*}{CC Power Plant} & \multirow{4}{*}{9,568} & \multirow{4}{*}{4}
  & RS-ORT      & \textbf{6.34}  & \textbf{6.32}  & \textbf{$<$0.01} & \textbf{3380.00}  \\
 & & & Bertsimas  & 6.35  & 6.36  & 98.90     & 14400.00 \\
 & & & OSRT       & OoM   & OoM   & OoM     & OoM     \\
 & & & CART       & 6.35  & 6.36  & --      & --     \\
\midrule
\multirow{4}{*}{Steel Industry} & \multirow{4}{*}{35,040} & \multirow{4}{*}{8}
  & RS-ORT      & \textbf{7.66}  & \textbf{7.61}  & \textbf{$<$0.01} & \textbf{1980.00}  \\
 & & & Bertsimas  & 8.42  & 8.33  & 100.00  & 14400.00 \\
 & & & OSRT       & OoM   & OoM   & OoM     & OoM      \\
 & & & CART       & 8.42  & 8.33  & --      & --     \\
\midrule
\multirow{4}{*}{Micro Gas Turbine} & \multirow{4}{*}{71,225} & \multirow{4}{*}{2}
  & RS-ORT        & 352.10  & 340.65  & \textbf{$<$0.01} & \textbf{1859.00} \\
 & & & Bertsimas  & 352.10 & 340.65 & 100.00     & 14400.00 \\
 & & & OSRT       & TOWT    & TOWT    & TOWT    & TOWT    \\
 & & & CART       & 352.10  & 340.65  & --      & --   \\
\midrule
\multirow{4}{*}{Household } & \multirow{4}{*}{2,049,280} & \multirow{4}{*}{5}
  & RS-ORT      & \textbf{0.29}  & \textbf{0.29}  & \textbf{33.47} & \textbf{14400.00} \\
 & & & Bertsimas  & OoM  & OoM  & OoM     & OoM    \\
 & & & OSRT       & TOWT   & TOWT   & TOWT      & TOWT     \\
 & & & CART       & 0.30  & 0.30  & --       & --    \\
\bottomrule
\end{tabular}
\end{table*}

As shown in Table~\ref{tab:test-results-continuous}, $n$ denotes the number of samples and $P$ the number of input features for each dataset.
In the table we report the results of our RS-ORT model, the MIQP formulation by \cite{bertsimas_regression_2017}, OSRT method, and CART on continuous benchmark datasets. This setting allowed us to verify that RS-ORT consistently achieved provably optimal solutions, as evidenced by the negligible training–testing optimality gaps (below 0.01\%) across multiple datasets. The results are consistent with those reported by other exact methods, supporting the reliability of our implementation. All methods were evaluated with a regularization parameter of 0.0005, chosen to ensure comparability by forcing all methods to grow full trees.

Bertsimas’s formulation exhibited large reported optimality gaps, which primarily reflect time-limited runs that did not close the MIP gap. This is an expected outcome given the exponential growth of full-space formulations as the number of samples increases.

OSRT encountered out-of-memory (OoM) errors, particularly on large-scale dataset with features that has a wide variety of possible values (e.g., the Steel Industry dataset). In some cases, it also failed to produce any valid solution within the allowed time, marked as TOWT (time-out without tree). In such cases, the method was not able to complete execution within the time wall and did not return a valid tree; therefore its performance metrics were thus omitted from the comparison. OSRT's observed OoM and TOWT outcomes are consistent with OSRT’s binary encoding on dense, high-cardinality continuous features, which can inflate the design of binary columns and exhaust memory or fail to return a valid tree within the time limit.

These limitations underscore the benefits of our reduced-space formulation, which achieved global optimality on continuous-valued features and demonstrated greater scalability across large and high-dimensional datasets.

\begin{figure*}[t]
    \centering
    \begin{subfigure}[b]{0.48\textwidth}
        \centering
        \includegraphics[width=\linewidth]{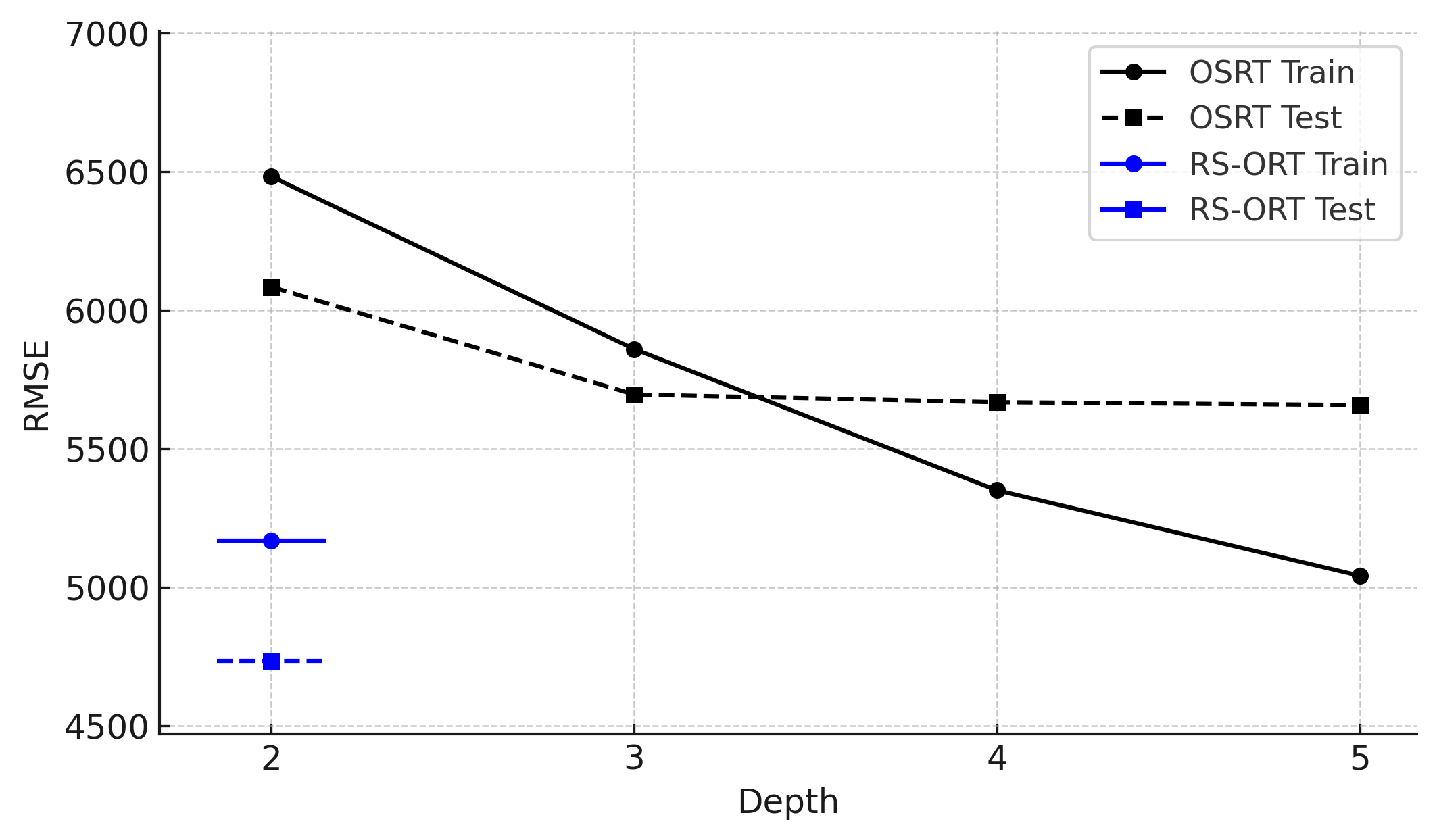}
        \caption{\emph{Insurance}}
        \label{fig:osrt_rs-ort_insurance}
    \end{subfigure}
    \hfill
    \begin{subfigure}[b]{0.48\textwidth}
        \centering
        \includegraphics[width=\linewidth]{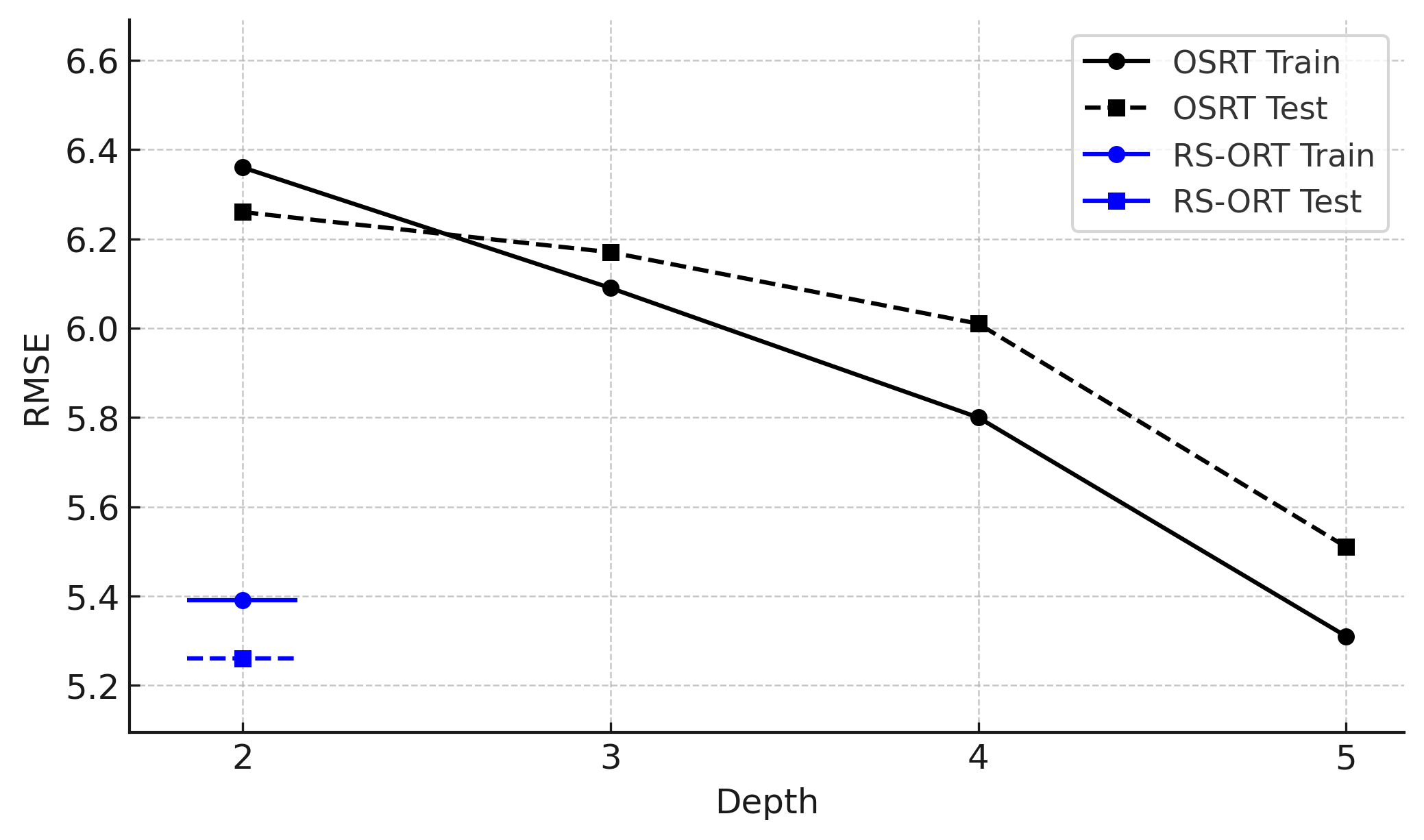}
        \caption{\emph{Airfoil Noise}}
        \label{fig:osrt_rs-ort_airfoil}
    \end{subfigure}

    \vspace{1em}

    \begin{subfigure}[b]{0.48\textwidth}
        \centering
        \includegraphics[width=\linewidth]{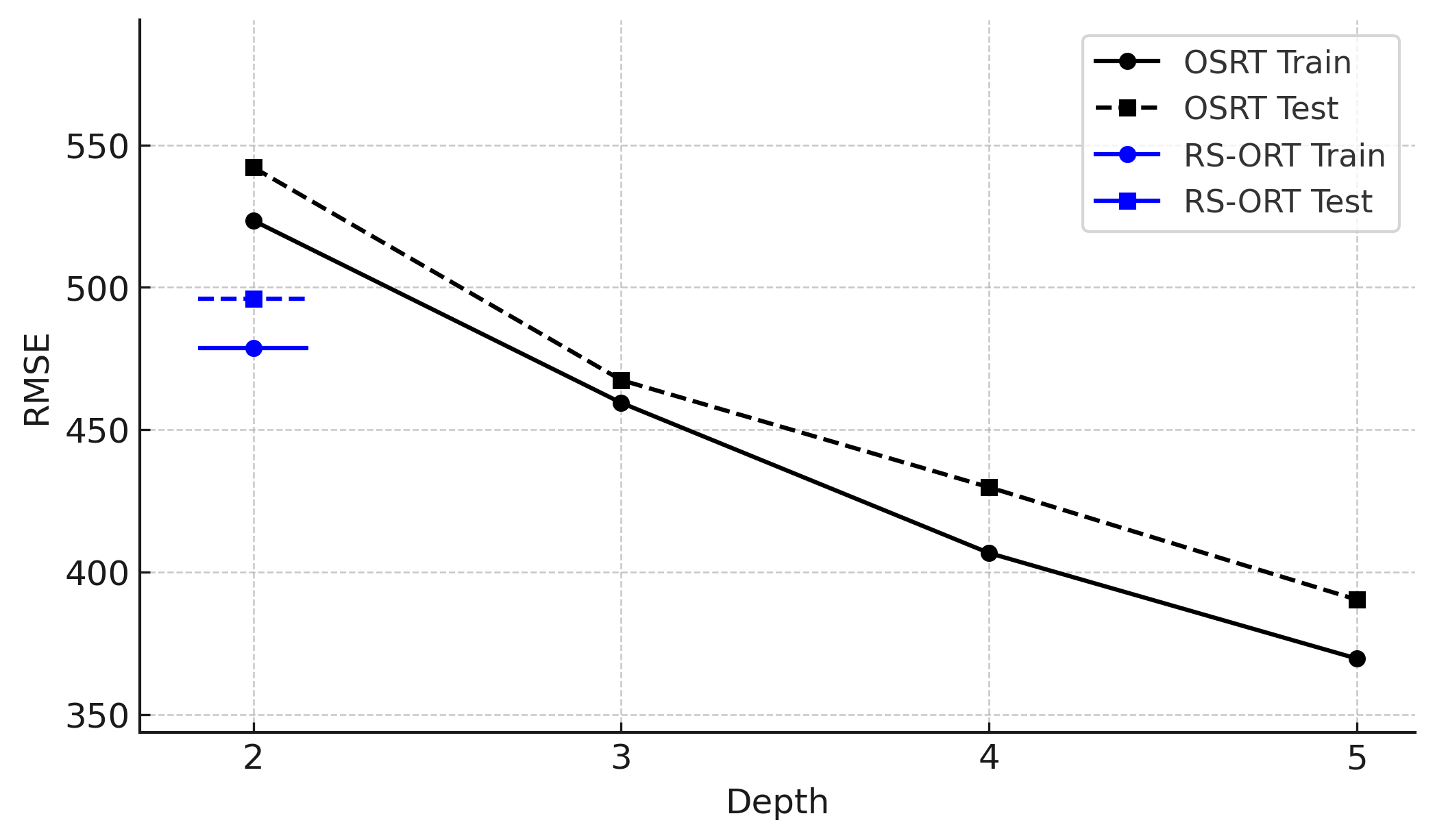}
        \caption{\emph{Seoul Bike}}
        \label{fig:osrt_rs-ort_seoul}
    \end{subfigure}
    \hfill
    \begin{subfigure}[b]{0.48\textwidth}
        \centering
        \includegraphics[width=\linewidth]{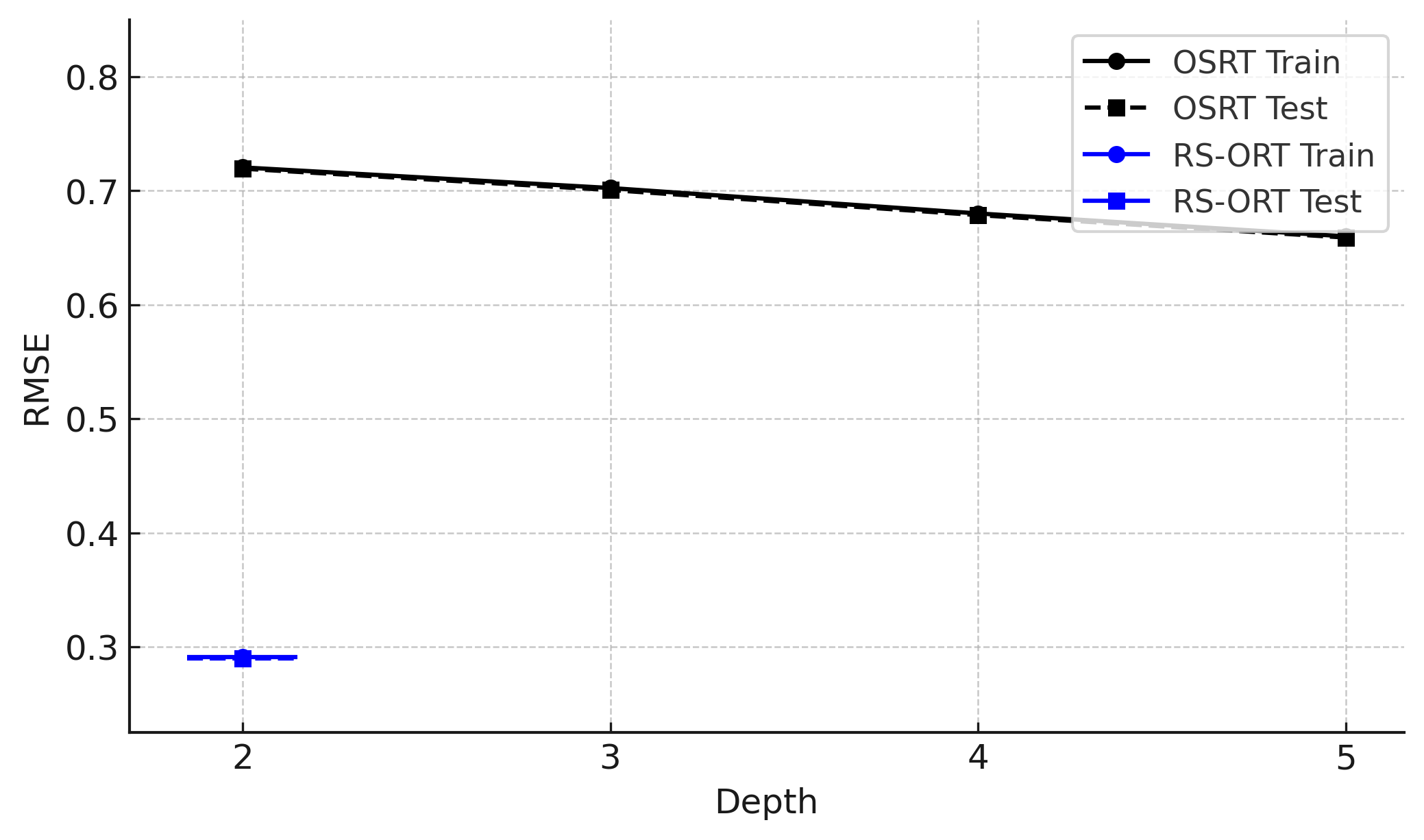}
        \caption{\emph{Household}}
        \label{fig:osrt_rs-ort_household}
    \end{subfigure}

    \caption{Train and test RMSE comparison between OSRT (binarized input) and RS-ORT (continuous input) across four datasets.}
    \label{fig:osrt_rs-ort_composite}
\end{figure*}

To complement our earlier analysis—which demonstrated RS-ORT’s strength on raw continuous features—we also measured how well it performs when its competitors are given their preferred binarised inputs. Specifically, we replicated the experimental protocol used by OSRT, whose formulation presupposes binary feature encodings. We selected the public datasets already distributed in binarised form by the OSRT Python package, then trained OSRT across tree depths 2 – 5 and over a hyper-parameter grid 
($\lambda \in {0.0001, 0.001, 0.005, 0.01, 0.05}$), retaining the model with the lowest training RMSE. RS-ORT, by contrast, was kept at a fixed depth 2 and consumed the original continuous features without any discretization.

Each plot in Figure~\ref{fig:osrt_rs-ort_composite} shows OSRT’s training and test RMSE across depths 2 to 5 (black lines with circle and square markers, respectively), while RS-ORT appears as a horizontal blue line at depth 2, with solid and dashed segments for training and test RMSE. In these experiments, each method is evaluated in its natural setting: OSRT uses its original binarized inputs as specified by its authors, whereas RS-ORT operates directly on the continuous features. This comparison clearly illustrates how RS-ORT behaves under its intended continuous formulation relative to a well-established binarized approach.

Despite OSRT’s structural advantage and access to deeper trees, RS-ORT performs remarkably well across all datasets. On \emph{Insurance}, RS-ORT achieves both lower training and test RMSE than any OSRT depth. On \emph{Airfoil Noise}, RS-ORT slightly outperforms OSRT in test RMSE and is close in training error. For \emph{Seoul Bike}, RS-ORT's performance is comparable to OSRT at depth 3. Finally, on \emph{Household}, RS-ORT clearly dominates, with a test RMSE less than half that of OSRT’s best configuration. These results demonstrate RS-ORT’s ability to generalize effectively while maintaining a compact and interpretable structure.

Overall, RS-ORT demonstrated superior performance across the majority of datasets, offering both strong generalization and optimality guarantees on continuous data, while reducing runtime and even complexity of models compared to existing exact methods.

\section{Conclusion}
This work introduces RS-ORT, a branch-and-bound algorithm designed for training optimal regression trees. By branching only on structural variables and computing leaf predictions in closed form, RS-ORT maintains a compact search space that scales independently of dataset size. To further improve efficiency, we discretize split thresholds over empirical feature values and resolve terminal splits using depth-1 CART models. These strategies allow our method to retain exactness while significantly reducing runtime. Our experiments show that RS-ORT consistently outperforms the MIQP-based baseline in both accuracy and solve time, achieving optimality on most datasets—even when the baseline fails to close the gap. On binarised benchmarks, RS-ORT reaches optimal solutions in minutes; however, we highlight that this comes at the cost of feature granularity. By handling continuous features directly, RS-ORT offers a more robust and interpretable alternative for regression tasks where precision and model transparency are essential.

\bibliography{cites}
\bibliographystyle{iclr2026_conference}
\newpage
\appendix
\begin{center}
    {\LARGE \textbf{Appendix}} \\[1.5ex]
\end{center}
\vspace{1em}
\section{Convergence Analysis of RS-ORT Algorithm}
\label{apdx:alg_conv}

We prove that the branch-and-bound (BB) scheme that \emph{branches only on the tree–structural variables} \(m=(a,b,c,d)\) converges to the global optimum of the RS-ORT algorithm.

\subsection{Two-stage reformulation}

Let the first-stage problem be
\[
f(M_0)\;=\;
  \min_{m\in M_0}\;
  \sum_{i\in\mathcal N} Q_i(m),
\]
where \(M_0\subset\mathbb R^{|m|}\) is the initial box for the structural variables \(m=(a,b,c,d)\).  

For fixed \(m\), the second-stage problem for sample \(i\) is
\[
Q_i(m)\;=\;
  \min_{z_{it},\,L_{i*}}
       \frac{L_{i*}}{\hat L}
       +\frac{\lambda}{n}\sum_{t\in\mathcal T_D}d_t
  \quad\text{s.t.\ (2b)–(2m).}
\]
This subproblem determines the optimal routing \(z_{it}\) and squared loss \(L_{i*}\) for sample \(i\) under a given tree structure \(m\).
The BB algorithm partitions subregions \(M\subseteq M_0\) in the reduced space of first-stage variables. The size of this space depends only on the tree depth \(D\) and number of features \(P\), and is therefore independent of the number of samples \(n\).
At each node \(M\), we compute a lower bound \(\beta(M)\) by relaxing the non-anticipativity constraints across samples, and an upper bound \(\alpha(M)\) by evaluating a feasible global solution from a candidate \(m\in M\). These bounds satisfy
\[
  \beta(M)\;\le\; f(M)\;\le\;\alpha(M).
\]
Branching proceeds by fixing or narrowing the domain of one coordinate in \(m\), producing two disjoint children \(M_1, M_2\) such that \(\operatorname{diam}(M_j)<\operatorname{diam}(M)\), where \(\operatorname{diam}(M)\) denotes the diameter of region \(M\). This ensures finite subdivision and prevents cycles.

\subsection{Convergence of the global bounds}

We now prove that the bounds computed at each iteration converge to the global optimum \(f^\star := f(M_0)\).

\paragraph{Convergence of the lower bound \(\beta_t\).}
At iteration \(t\), let \(\mathcal M_t\) denote the set of active regions and define the global lower bound as \(\beta_t := \min_{M \in \mathcal M_t} \beta(M)\).  
Because the optimal solution \(m^\star \in M_0\) always remains in some active region \(M^\star_t \in \mathcal M_t\), we have \(\beta_t \le \beta(M^\star_t) \le f(M^\star_t) = f^\star\).  
Moreover, the lower-bounding operator used in our algorithm is \emph{strongly consistent} in the sense of~\cite[Def.~4]{cao_scalable_2019}, and branching ensures an exhaustive partition of the domain. By Lemma 3 of~\cite{cao_scalable_2019} (extended by~\cite{hua_scalable_2022}), this implies
\[
  \lim_{t \to \infty} \beta_t = f^\star.
\]

\paragraph{Convergence of the upper bound \(\alpha_t\).}
The global upper bound is updated at each iteration as the minimum objective value among all feasible solutions found so far:
\[
\alpha_t := \min_{k \le t}\;\sum_i Q_i(m_k),
\]
where \(m_k\) is any candidate tree structure explored. Once all binary variables in \(m = (a,b,c,d)\) are fixed, only the continuous thresholds remain, and each subproblem \(Q_i(m)\) becomes the pointwise minimum of finitely many quadratic functions. Hence each \(Q_i(m)\) is Lipschitz-continuous in the continuous variables, and so is the global objective. By Lemma 6 of~\cite{cao_scalable_2019}, shrinking partitions around \(m^\star\) yield feasible evaluations approaching the optimum, which ensures
\[
  \lim_{t \to \infty} \alpha_t = f^\star.
\]

\paragraph{Global convergence.}
The sequences \(\beta_t\) and \(\alpha_t\) are monotone by construction: \(\beta_t\) is non-decreasing, \(\alpha_t\) is non-increasing, and always \(\beta_t \le \alpha_t\). Combining the two limits above, we conclude
\[
  \boxed{
    \lim_{t\to\infty} \beta_t = 
    \lim_{t\to\infty} \alpha_t = 
    f^\star.
  }
\]
Therefore, the reduced-space BB algorithm terminates when the gap \(|\alpha_t - \beta_t|\) falls below a user-defined tolerance \(\varepsilon\), returning an \(\varepsilon\)-optimal solution. Since all branching is performed over the fixed-size space of tree-structural variables \(m\), convergence is guaranteed independently of the number of training samples \(n\).

\section{Optimality of Second-Stage Decisions Given a Fixed Tree Structure}
\label{apdx:bb_thm}
We prove that once the tree structure \(m = (a, b, c, d)\) is fixed, the second-stage decisions—namely, the routing of samples and the computation of their associated losses—can be determined optimally and independently for each sample.
\subsection{Problem Setting}
Given a fixed tree structure \(m\), all split variables, thresholds, and active nodes are known. In this setting, the first-stage decisions are no longer optimization variables; they define a deterministic routing function from features \(x_i\) to terminal leaves \(t\). The goal is to solve the following second-stage problem:
\[
f(m) \;=\; \sum_{i\in\mathcal N} Q_i(m),
\]
where each subproblem \(Q_i(m)\) for sample \(i\) is defined as
\[
Q_i(m)\;=\;
  \min_{z_{it},\,L_{i*}}
       \frac{L_{i*}}{\hat L}
       +\frac{\lambda}{n}\sum_{t\in\mathcal T_D}d_t
  \quad\text{s.t.\ (2b)–(2l).}
\]
\subsection{Decoupling Across Samples}
Once \(m\) is fixed, all branching logic becomes static. The routing conditions for each sample \(i\) reduce to a deterministic set of logical rules based solely on the features \(x_i\). This implies that constraints (2b)–(2k) only involve sample-local variables \((z_{it}, L_{i*})\), and can be evaluated without reference to any other sample \(j \ne i\).
Consequently, the second-stage problem decomposes into \(n\) independent subproblems, one per sample. Each subproblem involves selecting the unique terminal leaf \(t\) to which sample \(i\) is routed, setting \(z_{it} = 1\) for that leaf and \(z_{it'} = 0\) for all others, and computing the loss \(L_{i*} = (y_i - c_t)^2\), where \(c_t\) is the prediction assigned to that leaf.
\subsection{Existence and Optimality of Local Solutions}
Because each subproblem \(Q_i(m)\) involves a finite set of routing options and a quadratic loss (which is convex in \(c_t\)), it admits a closed-form solution when \(c_t\) is allowed to vary. In particular, if the prediction \(c_t\) is optimized jointly after assigning samples, then the optimal value is the mean of the responses \(y_i\) for all \(i\) assigned to leaf \(t\). Alternatively, if \(c_t\) is fixed as part of \(m\), then the loss is directly computable.
In both cases, the subproblem is either trivially determined (by logical routing and evaluation) or solvable exactly (by closed-form minimization over a convex function on a fixed set). Hence, for any given \(m\), the optimal second-stage decisions for all samples can be computed exactly and independently.
\subsection{Conclusion}
We conclude that fixing the first-stage variables \(m = (a, b, c, d)\) reduces the RS-ORT problem to \(n\) fully separable and tractable subproblems, one for each sample. The global objective value \(f(m) = \sum_i Q_i(m)\) can therefore be computed exactly. This validates the correctness of the reduced-space formulation and confirms that second-stage optimality is preserved under structural decomposition.

\section{Branching Strategy for Algorithm}\label{apdx:bb_bch}

The branching strategy in Algorithm~\ref{alg:bb_sche} follows a strict variable priority to guide the search efficiently while preserving the structure of the tree. The order of branching is:

\begin{enumerate}

    \item \textbf{Split decisions ($d$):} These binary variables determine whether each internal node splits or becomes a leaf. The algorithm always branches on the first unfixed $d$-variable corresponding to a parent node (i.e., among the first $2^{D-1} - 1$ nodes). This ensures that the overall skeleton of the tree is fixed early.
    
    \item \textbf{Split features ($a$):} Once all relevant $d$-variables are fixed, the algorithm proceeds to assign split features. It selects the first grandparent unfixed $a_{j,t}$ variable, preferring nodes closer to the root.
    
    \item \textbf{Split thresholds ($b$):} If all $a$ grandparent variables for branching nodes are assigned, the algorithm branches on continuous threshold variables. For each node $t$, it computes a weighted uncertainty score:
    \[
    \Delta_t = (b_t^u - b_t^l) \cdot w_t,
    \]
    where $b_t^l$ and $b_t^u$ denote the current lower and upper bounds of the feasible interval for threshold $b_t$, and $w_t$ is a weight that decreases with the depth of node $t$ (i.e., higher weights near the root). Among all internal nodes, the $b_t$ with the largest $\Delta_t$ is selected.

\end{enumerate}

The leaf predictions $c$ are excluded from the branching process. Once the entire tree structure is determined (i.e., all $(d, a, b)$ variables are fixed), the values $c_t$ can be computed exactly via closed-form expressions.

\section{Proof of Error Decomposability}
\label{apdx:proof-error-decomp}
Consider a tree model with one parent node \(P\) and two child nodes \(L\) and \(R\). 
The global error for this tree can be written as:
\[
\sum_{i \in P}(y_i - \hat{y}_P)^2 
- \biggl[\sum_{i \in L}(y_i - \hat{y}_L)^2 + \sum_{i \in R}(y_i - \hat{y}_R)^2\biggr],
\]
where \(y_i\) is the target value of the \(i\)-th observation, and \(\hat{y}_L\), \(\hat{y}_R\), \(\hat{y}_P\) are the corresponding mean predictions in regions \(L\), \(R\), and \(P\), respectively.

Because each summation is structurally similar, it suffices to show that one of them is decomposable in order to conclude that the entire sum is decomposable. Thus, we focus on proving that \(\sum (y_i - \hat{y})^2\) can be decomposed appropriately.

\textbf{Theorem (Decomposability of \(\sum (y_i - \hat{y})^2\)).}
Let \(\{y_1, y_2, \ldots, y_n\}\) be a set of \(n\) real values, and let 
\(\hat{y} = \tfrac{1}{n}\sum_{i=1}^{n}y_i\). Then
\[
\sum_{i=1}^{n} (y_i - \hat{y})^2 = \sum_{i=1}^{n}y_i^2 - n\,\hat{y}^2.
\]
Moreover, this quantity can be decomposed into subgroups of data.\\
\textit{Proof.}
First, expand the square:
\[
\sum_{i=1}^{n} (y_i - \hat{y})^2
= \sum_{i=1}^{n} \bigl(y_i^2 - 2\,y_i\,\hat{y} + \hat{y}^2\bigr).
\]
Distributing the summation, we obtain:
\[
\sum_{i=1}^{n} y_i^2 - 2\sum_{i=1}^{n} (y_i\,\hat{y}) + \sum_{i=1}^{n} \hat{y}^2.
\]
Since \(\hat{y}\) is a constant, we get:
\[
\sum_{i=1}^{n} y_i^2 - 2\,\hat{y}\sum_{i=1}^{n} y_i + n\,\hat{y}^2.
\]
Recall that \(\hat{y} = \tfrac{1}{n}\sum_{i=1}^{n}y_i\), so:
\[
\sum_{i=1}^{n} y_i = n\,\hat{y}.
\]
Substitute this into the expression:
\[
\sum_{i=1}^{n} y_i^2 - 2\,n\,\hat{y}^2 + n\,\hat{y}^2
= \sum_{i=1}^{n} y_i^2 - n\,\hat{y}^2.
\]
Thus:
\[
\sum_{i=1}^{n} (y_i - \hat{y})^2 = \sum_{i=1}^{n} y_i^2 - n\,\hat{y}^2.
\]
\textbf{Remark.}
Because \(\sum_{i=1}^{n} y_i^2\) and \(n\,\hat{y}^2\) decompose naturally over disjoint groups of data, the same holds for 
\(\sum_{i=1}^{n} (y_i - \hat{y})^2\). If data is partitioned into \(m\) groups \(\{G_1, G_2, \ldots, G_m\}\), then the global sum of squares decomposes into the sums within each group.

\section*{Proof of closed-form optimality of leaf predictions}
\label{apdx:proof-closed-form-leaf}
\begin{theorem}[Closed-form optimality of leaf predictions]

Let \( \mathcal{S}_t = \{ i \in \mathcal{N} : z_{it} = 1 \} \) denote the set of samples assigned to leaf \( t \in \mathcal{T}_L \), and assume the routing variables \( z \) are fixed. Under the squared-error loss, the unique value of \( c_t \) that minimizes
\[
\sum_{i \in \mathcal{S}_t} (y_i - c_t)^2
\]
is given by the empirical mean of its targets:
\[
c_t^\star = \frac{1}{|\mathcal{S}_t|} \sum_{i \in \mathcal{S}_t} y_i.
\]

\end{theorem}

\begin{proof}
Define the function
\[
f(c) = \sum_{i \in \mathcal{S}_t} (y_i - c)^2.
\]
We differentiate with respect to $c$:
\[
f'(c) = \sum_{i \in \mathcal{S}_t} 2(c - y_i)
= 2\Bigl(|\mathcal{S}_t|\,c - \sum_{i \in \mathcal{S}_t} y_i\Bigr).
\]
Setting the derivative equal to zero yields
\[
|\mathcal{S}_t|\,c = \sum_{i \in \mathcal{S}_t} y_i,
\]
and therefore
\[
c = \frac{1}{|\mathcal{S}_t|} \sum_{i \in \mathcal{S}_t} y_i.
\]
Finally, since
\[
f''(c) = 2|\mathcal{S}_t| > 0,
\]
the function is strictly convex, so this critical point is the unique minimizer.
\end{proof}

\end{document}